\documentclass[11pt]{article} % use larger type; default would be 10pt

\usepackage[latin1]{inputenc}

\usepackage{geometry} % to change the page dimensions
\usepackage{graphicx} % support the \includegraphics command and options
\usepackage{booktabs} % for much better looking tables
\usepackage{array} % for better arrays (eg matrices) in maths
\usepackage{paralist} % very flexible & customisable lists (eg. enumerate/itemize, etc.)
\usepackage{verbatim} % adds environment for commenting out blocks of text & for better verbatim
\usepackage{subfig} % make it possible to include more than one captioned figure/table in a single float
\usepackage{amssymb}
\usepackage{bbm}
\usepackage{stmaryrd}
\usepackage{mathrsfs}
\usepackage{amsfonts,amssymb,amsmath,amsthm}
\usepackage{hyperref}

\usepackage{fancyhdr} % This should be set AFTER setting up the page geometry
\usepackage{sectsty}
\allsectionsfont{\sffamily\mdseries\upshape} % (See the fntguide.pdf for font help)
\usepackage[nottoc,notlof,notlot]{tocbibind} % Put the bibliography in the ToC
\usepackage[titles,subfigure]{tocloft} % Alter the style of the Table of Contents

\title{Corrected Kriging update formulae for batch-sequential data assimilation}%{About the Kriging Update formulae}
\author{Cl\'ement Chevalier\footnote{Institut de Radioprotection et de S\^uret\'e Nucl\'eaire (IRSN), Fontenay-aux-Roses, France}, \\
David Ginsbourger\footnote{IMSV, Department of Mathematics and Statistics, University of Berne, Switzerland} %Alpeneggstrasse, 22, CH-3014, Berne, Switzerland}
}
\date{March 28th, 2012} % Activate to display a given date or no date (if empty),
\newcommand{\esp}{\mathbb{E}}									%expetation
\newcommand{\var}{\mathnormal{Var}}						%variance

\newcommand{\zhat}[2]{\hat{Z}(#2)_{#1\ \text{data}}}

\newcommand{\sigmahatsquare}[2]{\sigma^2_{#1}(#2)}

\newcommand \trans {^\top}
\newtheorem{prop}{Proposition}

\newtheorem{cor}{Corollary}

\newcommand{\x}{\mathbf{x}}	
\newcommand{\y}{\mathbf{y}}
\newcommand{\Xn}{\mathbf{X}_{\text{new}}}
\newcommand{\Xo}{\mathbf{X}_{\text{old}}}
\newcommand{\zn}{\mathbf{Z}_{\text{new}}}
\newcommand{\zo}{\mathbf{Z}_{\text{old}}}
\newcommand{\mn}[1]{\hat{Z}(#1)_{\text{new}}}%m_{\text{old}}}	
\newcommand{\mo}[1]{\hat{Z}(#1)_{\text{old}}}%m_{\text{old}}}	
\newcommand{\vn}{\sigma^{2}_{\text{new}}}%{s^{2}_{\text{new}}}	
\newcommand{\vo}{\sigma^{2}_{\text{old}}}
\newcommand{\Kn}{\Sigma_{\text{new}}}

\newcommand{\wnn}{\boldsymbol{\lambda}_{\text{new,new}}}	
\newcommand{\wno}{\boldsymbol{\lambda}_{\text{new,old}}}	
\newcommand{\ko}{\sigma_{\text{old}}}
\newcommand{\kn}{\sigma_{\text{new}}}

\begin{document}
\maketitle

\section{Introduction}

Recently, a lot of effort has been paid to the efficient computation of Kriging predictors when observations are assimilated sequentially. 
In particular, Kriging update formulae enabling significant computational savings were derived in \cite{barnes1992efficient, gao:1996:update,emery:2009:update}. 
Taking advantage of the previous Kriging mean and variance calculations helps avoiding a costly $(n+1) \times (n+1)$ matrix inversion when adding one observation to the $n$ already available ones. 
In addition to traditional update formulae taking into account a single new observation, \cite{emery:2009:update} also proposed formulae for the batch-sequential case, i.e. when $r > 1$ new observations are simultaneously assimilated.
However, the Kriging variance and covariance formulae given without proof in \cite{emery:2009:update} for the batch-sequential case are not correct. 
In this paper we fix this issue and establish corrected expressions for updated Kriging variances and covariances when assimilating several observations in parallel. 
%We give a counter example for \cite{emery:2009:update}'s formula and explain in detail the differences between the corrected formulae and the formulae proposed in \cite{emery:2009:update}. 

%Emery a fait un papier bien, mais il a donne une extension dans le cas parallel sans preuve. On a essaye de l'appliquer et elle est %fausse. cette note a pour but de ... on essaye de .... on enonce la nouvelle prop. et on la prouve et on exhibe la diff entre la nouvelle %formule et celle d'emery et on verifie qu'elle est bonne sur le contre exemple...

\section{Kriging update formulae for the parallel case}
\subsection{notations and formulae}

In this section we give a counter-example for the Kriging variance update formula proposed in \cite{emery:2009:update}. 
Let us first introduce the notations of this paper and recall these formulae. 
When $n$ observations, at locations $\x_1,\hdots,\x_n \in D$, of a real valued random field $Z$ indexed by $D\subset \mathbb{R}^{d}$ are available,  
\cite{emery:2009:update} denotes respectively by $\zhat{n}{\x}$ and $\sigmahatsquare{n}{\x}$ the Kriging mean and variance at $\x \in D$ based  
on these $n$ observations. 
The corresponding Kriging covariance is denoted by $\sigma_{n}: (\x,\mathbf{y}) \in D^2 \to \sigma_{n}(\x,\mathbf{y})$.
If $k > 1$ new observations, at points $\x_{n+1},\hdots,\x_{n+k}$, are available, the parallel Kriging update formulae given in \cite{emery:2009:update} follow:
\begin{align}  
\label{update:emerymean}
\zhat{n+k}{\x}  \;=&\; \zhat{n}{\x} \;+\; 
\sum_{i=1}^k \lambda_{n+i|n+k}(\x) %\, \,\cdot\, 
\left( Z(\x_{n+i}) - \zhat{n}{\x_{n+i}}  \right), \\
\label{update:emeryvar}
\sigmahatsquare{n+k}{\x}  \;=&\; \sigmahatsquare{n}{\x} \;-\; 
\sum_{i=1}^k \lambda^2_{n+i|n+k}(\x) 
\sigmahatsquare{n}{\x_{n+i}}\ , \\
\label{update:emerycov}
\sigma_{n+k}(\x,\mathbf{y})  \;=&\; \sigma_{n}(\x,\mathbf{y}) \;-\; 
\sum_{i=1}^k \lambda_{n+i|n+k}(\x)\lambda_{n+i|n+k}(\mathbf{y}) 
\sigmahatsquare{n}{\x_{n+i}}\ 
\end{align}
where $\lambda_{n+i|n+k}(\x)$ denotes the Kriging weight of $Z(\x_{n+i})$ when predicting $Z(\x)$ relying on $n+k$ observations. 
In \cite{emery:2009:update}, these formulae are proven only for $k = 1$.
For  $k > 1$, they are quickly justified with an argument based on Pythagoras' theorem. 
In fact, the assumptions for using this theorem are not satisfied, as we will detail here. 
%which explains why formula~(\ref{update:emeryvar}) and (\ref{update:emerycov}) do not hold. 
We now give a counter-example for Equation~(\ref{update:emeryvar}) %and~(\ref{update:emerycov})  
and then propose fixed expressions. 

\subsection{A counterexample for Equation~(\ref{update:emeryvar}) }

\noindent
Let us consider for the sake of simplicity a case where $Z$ is a one-dimensional centered random process indexed by $D=[0,1]$, 
with covariance kernel $C(x,y):=\sigma_{0}(x,y)=\min(x,y)$ (Wiener process, or \textit{Brownian Motion}). Note that even though this 
process is non-stationary, Simple Kriging equations are applicable to it (Written in the general case of a non-stationary 
covariance kernel, as allowed in \cite{emery:2009:update}'s results and notations). Assuming $n = 0$ initial observations and $k = 2$ 
new observations at the points $x_1 = \frac{1}{2}$,  $x_2 = 1$, the Kriging weights for a prediction at $x = \frac{3}{4}$ write:
\begin{align*}
(\lambda_{1|2}(x),\lambda_{2|2}(x)) =& \ (C(x,x_1),C(x,x_2) ) 
%K^{-1}
\begin{pmatrix}
      C(x_1,x_1) & C(x_1,x_2)\\
      C(x_2,x_1) & C(x_2,x_2)\\
    \end{pmatrix}^{-1}
%= \ \left(\frac{1}{2} , \frac{3}{4}\right) \begin{pmatrix}
%      4 & -2\\
%      -2 & 2\\
%    \end{pmatrix}
= \ \left(\frac{1}{2},\frac{1}{2}\right)
\end{align*}
leading to a Kriging mean $\zhat{n+k}{\x}=\frac{1}{2} Z(x_1) + \frac{1}{2} Z(x_2)$, and to a Kriging variance: 
\begin{align*}
\sigmahatsquare{2}{x} = C(x,x) - (\lambda_{1|2}(x),\lambda_{2|2}(x)) 
\begin{pmatrix}
      C(x_1,x_1) & C(x_1,x_2)\\
      C(x_2,x_1) & C(x_2,x_2)\\
    \end{pmatrix}
(\lambda_{1|2}(x),\lambda_{2|2}(x))^{T} = \frac{1}{8}
\end{align*}
\noindent
Now, using Eq.~(\ref{update:emeryvar}) would lead to 
\begin{align*}
\sigmahatsquare{2}{x} 
&= \sigmahatsquare{0}{x} - \lambda^2_{1|2}(x) \sigmahatsquare{0}{x_1} - \lambda^2_{2|2}(x) \sigmahatsquare{0}{x_2} \\
&= \frac{3}{4} - \frac{1}{4}\frac{1}{2} -\frac{1}{4} = \frac{3}{8}
\end{align*}
\noindent
and we see that the two expressions do not match. More precisely, the conditional variance obtained by using Eq.~(\ref{update:emeryvar}) 
is too large, as if the contribution of $x_1$ and $x_2$ to decreasing the variance would have been underestimated. As we will see 
now, this is indeed the case, and the missing part is related to conditional covariances. 

\subsection{Corrected Kriging variance and covariance update formulae}

%Transition avant nos resultats: cadre GRF (gaussian random field) + introduction nos notations. On raisonne en terme de PG. On transcrit en terme de notation Emery. C'est technique, on utilise nos notations (decharge un peu emery)\\
%proprietes, bonnes formules \\
%dans la partie ou on reformule les notations: en profiter pr expliquer l'histoire du gaussien (subtilite: pas evident). rappeller ce que sont les poids de krigeage et donner les cles de comprehension de la preuve. la moyenne de krigeage est tjs une CL des obs. au temps n ... au temps new.... ds le cas gaussien ca coincide avec esperance cond et variance cond.

In this section, we establish corrected update formulae that may be used in the batch-sequential case instead of the formulae~(\ref{update:emeryvar}) and~(\ref{update:emerycov}). 
%We propose to change the notations in order to emphasize the Gaussian random field framework used for these formulae. 
To improve the readability of the properties and their proofs, we adopt the following simplified notations: 

\begin{itemize}
	\item $\Xo := \{\x_{1}, \ldots, \x_{n}\}$, and $\Xn := \{\x_{n+1}, \ldots, \x_{n+k}\}$,
	\item $\zo:=( Z(\x_{1}), \ldots, Z(\x_{n}) )$, and $\zn:=( Z(\x_{n+1}), \ldots, Z(\x_{n+k}) )$, 
%	\item $\wo(\x) := (\lambda_{1|n}(\x), \ldots, \lambda_{n|n}(\x))^{T}$, 
	\item $\wno(\x) := (\lambda_{1|n+k}(\x), \ldots, \lambda_{n|n+k}(\x))^{T}$,
	\item $\wnn(\x) := (\lambda_{n+1|n+k}(\x), \ldots, \lambda_{n+k|n+k}(\x))^{T}$, 
	\item $\vo(\x):=\sigmahatsquare{n}{\x}$, $\vn(\x):=\sigmahatsquare{n+k}{\x}$, and similarly for the conditional covariances.
	\item The conditional covariance matrix of $\zn$ knowing $\zo$ is denoted by $\Kn$,
\end{itemize}

\noindent
For conciseness and coherence, $\zhat{n}{\x}$ and $\zhat{n+k}{\x}$ are also denoted by $\mo{\x}$ and $\mn{\x}$, respectively. 
%The Kriging mean and variance at time $n$ are denoted by $\mo{\x}$ and $\vo(\x)$ and, likewise, we use the notations $\mn{\x}$ and $\vn(\x)$ at time $n+k$.\\
\noindent
%With these notations, let us recall that the Kriging mean at some point $\x$ is the expectation of the random variable $Z(\x)$ conditionally on the observations. Such expectation is usually a linear combination of the observations.
The corrected update formulae are given below: 

%\bigskip
%\bigskip
%\hrule

\begin{prop} (Corrected Kriging update equations for the parallel case)
\begin{align}
\label{corrected_m}
\mn{\x} & = \mo{\x} + \wnn(\x)^{T} ( \zn - \mo{ \Xn } ) \\
\label{corrected_v}
\vn(\x) & = \vo(\x) - \wnn(\x)^{T} \Kn \wnn(\x) \\
\label{corrected_c}
\kn(\x,\y) & = \ko(\x,\y) - \wnn(\x)^{T} \Kn \wnn(\y)
\end{align}
\end{prop}

\noindent
The proofs above use a Gaussian assumption on the field $Z$, enabling a convenient interpretation of the 
Kriging mean and variance in terms of conditional expectation and variance. 
A key property is that even though the conditional distribution interpretation doesn't hold in non-Gaussian cases, 
the formulae for the Kriging weights remain valid whatever the assumed square-integrable distribution for the field; 
we are just using here that best linear prediction and conditioning coincide in the Gaussian case. 
Furthermore, the results hold as well for the cases of Ordinary and Universal Kriging (written in terms of covariances, 
for a square-integrable field, not with variograms), relying on their well-known Bayesian construction with an improper 
prior on the trend coefficient(s). 

\begin{proof} The first equation follows from an application of the law of total expectation:
\begin{equation*}
\begin{split}
\mo{\x} &= \esp[ Z(\x) | \zo ]\\
&= \esp[ \esp[ Z(\x) | \zo, \zn ] | \zo ] \\
&= \esp[ \wno(\x)^T\zo + \wnn(\x)^T\zn ] | \zo ] \\
&= \wno(\x)^T\zo + \wnn(\x)^T \mo{\Xn} \\
&= \wno(\x)^T\zo + \wnn(\x)^T\zn - \wnn(\x)^T\zn + \wnn(\x)^T \mo{\Xn} \\
&= \mn{\x} - \wnn(\x)^T( \zn - \mo{\Xn} )
\end{split}
\end{equation*}
\noindent
Similarly, using the law of total variance delivers:
\begin{equation*}
\begin{split}
\vo(\x) &= \var[ Z(\x) | \zo ]\\
&= \esp[ \var[ Z(\x) | \zo, \zn ] | \zo ] +  \var[ \esp[ Z(\x) | \zo, \zn ] | \zo ] \\
&= \var[ Z(\x) | \zo, \zn ] + \var[ \wno(\x)^{T} \zo  + \wnn(\x)^{T} \zn  | \zo ] \\
&= \vn(\x) +  \wnn(\x)^{T}\Kn \wnn(\x)
\end{split}
\end{equation*}
which proves the second equation. The third equation comes with the same method, using the law of total covariance.
\end{proof}

\begin{prop} (Kriging weights expressed in terms of conditional covariance)
\label{neweight}
\begin{align}
\Kn \wnn(\x) = \ko(\Xn, \x) 
\end{align}
\end{prop}

\begin{cor} (Corrected Kriging update equations in terms of conditional covariance)
\begin{align}
\label{corrected_m_bis}
\mn{\x} & = \mo{\x} + \ko(\Xn, \x)^{T} \Kn^{-1} ( \zn - \mo{ \Xn } ) \\
\label{corrected_v_bis}
\vn(\x) & = \vo(\x) - \ko(\Xn, \x)^{T} \Kn^{-1} \ko(\Xn, \x) \\
\label{corrected_c_bis}
\kn(\x,\y) & = \ko(\x,\y) - \ko(\Xn, \x)^{T} \Kn^{-1} \ko(\Xn, \y) 
\end{align}
\end{cor}

\begin{proof}
Using the orthogonal projection interpretation of the conditional expectation, % (Gaussian case),  
\begin{align*}
Z(\x) &= \esp(Z(\x)|\zo, \zn) + \overbrace{Z(\x) - \esp(Z(\x)|\zo, \zn)}^{=: \varepsilon}
\\
 &= \wno(\x)\trans \zo + \wnn(\x)\trans \zn + \varepsilon \ ,
\end{align*}
with $\varepsilon$ centered, and independent of $\zo$ and $\zn$. Let us now calculate the conditional covariance between $Z(\x)$ and $\zn$ knowing the observations $\zo$:
\begin{align*}
\ko(\Xn, \x):=& \  cov(\zn , Z(\x) |\zo )\\
=& \ 0 + cov\left(\zn, \wnn(\x) \trans \zn   \bigm|\zo\right) + cov( \zn, \varepsilon  | \zo) \\
=&  \Kn \wnn(\x) + cov(\varepsilon , \zn | \zo)
\end{align*}
\noindent
Noting that $cov(\zn, \varepsilon | \zo) = \mathbf{0}$, the latter equation proves Proposition~\ref{neweight}. 
%Hint de preuve (au cas ou demande par la review): ecrire epsilon comme l'esperance de epsilon sachant \zo + 1 reste indÃ©pendant
Eqs.~\ref{corrected_m_bis}, \ref{corrected_v_bis}, and \ref{corrected_v_bis} of the corollary directly follow by plugging in 
Eq.~\ref{neweight} into Eqs.~\ref{corrected_m}, \ref{corrected_v}, \ref{corrected_c}. %, corrected_v, corrected_c}. 
\end{proof}

\subsection{Interpretation on the counter-example}

The main difference between Equations~\ref{update:emeryvar} and \ref{corrected_v} is that Eq.~\ref{corrected_v} takes into account conditional Kriging covariances. % between the points $\Xn$. 
Coming back to the counter example, Eq.~\ref{corrected_v} delivers:
\begin{align*}
\sigmahatsquare{2}{x} 
&= \sigmahatsquare{0}{x} - \lambda^2_{1|2}(x) \sigmahatsquare{0}{x_1} - \lambda^2_{2|2}(x) \sigmahatsquare{0}{x_2} 
\mathbf{-2\lambda_{1|2}(x)\lambda_{2|2}(x)\sigma_{0}(x_1,x_2}) \\
&= \frac{3}{4} - \frac{1}{4}\frac{1}{2} -\frac{1}{4} -\mathbf{ 2 \frac{1}{4} \frac{1}{2}} %\\&
= {\frac{1}{8}}
\end{align*}
%$$\vn(\x) = \frac{3}{4} - \left(\frac{1}{2},\frac{1}{2}\right)\trans \begin{pmatrix} \frac{1}{2} & \frac{1}{2}\\ \frac{1}{2} & 1\\ \end{pmatrix} \left(\frac{1}{2},\frac{1}{2}\right) = 1/8 \  ,$$ 
which is the correct variance, as obtained earlier with regular Simple Kriging equations. 
Note that Eq.~\ref{update:emeryvar} would be correct if the matrix $\Kn$ were diagonal, which has no particular reason to happen in practical applications. 
%
%However, this is not the case in general, as the covariances between the Kriging residuals have no reason to be equal to zero. 
%
The Pytagorean theorem invoked in \cite{emery:2009:update}, enabling to decompose a variance into a sum of (weighted) variances, cannot hence be applied in the general case 
because of non-independence between Kriging residuals.

%Morale de l'histoire avec les covariances croisees. Faire apparaitre la diff entre notre formule et celle d'emery \\
%Conclusion de la note: on generalise. On a corrige la formule: chez emery il manquait les cov croisees $\sum_{i\neq j} ....$

%Emery a fait un papier bien, mais il a donne une extension dans le cas parallel sans preuve. On a essaye de l'appliquer et elle est fausse. cette note a pour but de ... on essaye de .... on enonce la nouvelle prop. et on la prouve et on exhibe la diff entre la nouvelle formule et celle d'emery et on verifie qu'elle est bonne sur le contre exemple...

\section{Conclusion}

%In the previous section we have proven (see: equations~\ref{corrected_m_bis}, \ref{corrected_v_bis} and \ref{corrected_c_bis}) that the update of the Kriging mean and variance at a point $\x$ based on new observations $\zn$ amounts to use the \emph{Simple Kriging equations} using  $\ko(.,.)$ as ``initial'' covariance kernel. We now detail the differences between our new expressions and the formulae~\ref{update:emeryvar} and \ref{update:emerycov}.

%conclusion: in this paper we.... (ce qu'on a fait). Rappeller l'erreur d'emery ds un cadre un peu general. s'applique au co-krigeage...
In this paper we derived corrected Kriging update formulae, fixing the incorrect formulae~\ref{update:emeryvar} and \ref{update:emerycov} given in \cite{emery:2009:update}, where 
a part of the update involving conditional covariances had been neglected. 
A simple interpretation of the new update formulae is that they correspond to Simple Kriging equations, where the underling covariance kernel would be the conditional covariance kernel before update, i.e $\ko(.,.)$. These formulae enable important computational savings (avoiding a large matrix inversion) and may perfectly be adapted to further frameworks such as co-Kriging.

% Je n'arrive pas à changer le style de bibliographie. P-ê en rapport avec le fait que les refs soient données en interne, pas dans un .bib?
%\bibliographystyle{abbrvnat}
%\bibliographystyle{amsalpha}%authordate1}

\bigskip
\noindent
\textbf{Acknowledgements:} The authors would like to thank Dr. Julien Bect (Ecole Sup\'erieure d'Electricit\'e) for drawing their attention to the conditional covariance interpretation of Kriging weights in the non-parallel case. Cl\'ement Chevalier acknowledges support from IRSN and the \href{http://www.redice-project.org/doku.php}{ReDICE consortium}. David Ginsbourger 
acknowledges support from the IMSV, University of Berne.

\end{document}